\documentclass{article}
\usepackage{kr}

\usepackage{times}
\usepackage{soul}
\usepackage{url}
\usepackage[hidelinks]{hyperref}
\usepackage[utf8]{inputenc}
\usepackage[small]{caption}
\usepackage{graphicx}
\usepackage{amsmath}
\usepackage{amsthm}
\usepackage{tikz}
\usepackage{multirow}
\usepackage{booktabs}
\usepackage{amssymb}
\usepackage{algorithm}
\usepackage{algorithmic}
\usepackage{comment}
\usepackage{array}
\usepackage{comment}
\usepackage{amsmath}
\usepackage{amsthm}

\usepackage{bm}
 \usepackage{booktabs}
 \usepackage{float}

\usepackage{algorithm}
\usepackage{algorithmic}
\usepackage{adjustbox}

\newtheorem{example}{Example}

\newtheorem{definition}{Definition}
\newtheorem{proposition}{Proposition}
\newtheorem{corollary}{Corollary}

\title{Enhancing Clustering: An Explainable Approach via Filtered Patterns}

\author{%
Motaz Ben Hassine$^1$\and
Sa\"id Jabbour $^1$
\affiliations
$^1$CRIL, University of Artois \& CNRS, Lens, France\\
\emails
\{benhassine,jabbour\}@cril.fr
}

\begin{document}

\maketitle

\begin{abstract}
Machine learning has become a central research area, with increasing attention devoted to explainable clustering, also known as conceptual clustering, which is a knowledge-driven unsupervised learning paradigm that partitions data into $\theta$ disjoint clusters, where each cluster is described by an explicit symbolic representation, typically expressed as a closed pattern or itemset. By providing human-interpretable cluster descriptions, explainable clustering plays an important role in explainable artificial intelligence and knowledge discovery.
Recent work improved clustering quality by introducing $k$-relaxed frequent patterns ($k$-RFPs), a pattern model that relaxes strict coverage constraints through a generalized $k$-cover definition. This framework integrates constraint-based reasoning, using SAT solvers for pattern generation, with combinatorial optimization, using Integer Linear Programming (ILP) for cluster selection. Despite its effectiveness, this approach suffers from a critical limitation: multiple distinct $k$-RFPs may induce identical $k$-covers, leading to redundant symbolic representations that unnecessarily enlarge the search space and increase computational complexity during cluster construction.
In this paper, we address this redundancy through a pattern reduction framework. Our contributions are threefold. First, we formally characterize the conditions under which distinct $k$-RFPs induce identical $k$-covers, providing theoretical foundations for redundancy detection. Second, we propose an optimization strategy that removes redundant patterns by retaining a single representative pattern for each distinct $k$-cover. Third, we investigate the interpretability and representativeness of the patterns selected by the ILP model by analyzing their robustness with respect to their induced clusters.
Extensive experiments conducted on several real-world datasets demonstrate that the proposed approach significantly reduces the pattern search space, improves computational efficiency, preserves and  enhances in some cases the quality of the resulting clusters.
\end{abstract}

\section{Introduction}
Clustering has been extensively studied in recent decades and is often formulated as an optimization problem. The objective is to identify clusters such that data points within the same cluster are highly similar, while those in different clusters are dissimilar. Depending on the objective function, the optimization task may involve maximizing intra-cluster similarity or minimizing inter-cluster similarity. The clustering problem is known to be $\mathcal{NP}$-hard, and its formulation as an optimization problem in the context of introduction the $k$-Means approach was first defined by \cite{Steinhaus_H_1956_j_bull_acad_polon_sci_division_cmp_k_means}.
Over the years, numerous approaches have been proposed to address the clustering problem. These methods can broadly be categorized into two main families based on how they handle cluster membership: \textit{overlapping} clustering methods, where data points may belong to multiple clusters simultaneously, and \textit{non-overlapping} (or disjoint) clustering methods, where each data point is assigned to exactly one cluster.
In overlapping clustering, the flexibility to share data points across clusters enables more nuanced representations of complex data structures. For instance, \cite{Bezdek1981PatternRW} introduced the fuzzy $c$-means algorithm, a generalization of $k$-Means \cite{macqueen1967multivariate,Steinhaus_H_1956_j_bull_acad_polon_sci_division_cmp_k_means}, where each data point is assigned a fuzzy degree of membership across multiple clusters rather than a hard assignment to a single cluster. Similarly, the Neo-$k$-Means approach \cite{whang2015non} represents an overlapping extension of the classical $k$-Means framework. Overlapping methods have proven particularly valuable in application domains such as social network analysis and biological network modeling, where entities naturally participate in multiple communities or functional groups. Notable examples include the clique percolation method \cite{ahn2010link}, the EAGLE approach \cite{shen2009detect}, and the COPRA method \cite{gregory2010finding}. More recent overlapping clustering techniques continue to advance the field \cite{ZHAO2024108104,ma2025fcpca,ahmed2025signature,zheng2025spatiotemporal}.
In contrast, non-overlapping clustering methods enforce strict partitioning, where clusters form disjoint subsets of the data. This paradigm is fundamental to many classical clustering algorithms and remains widely used due to its simplicity and interpretability. Prominent approaches in this category include the BIRCH method \cite{zhang1996birch} and various hierarchical agglomerative clustering techniques \cite{ackermann2014analysis,ward1963hierarchical}, along with numerous recent advances \cite{tao5201549deep,11045895,LIN2025127560,cao2025solving}.
An important extension of the disjoint clustering paradigm is \textbf{conceptual clustering}, introduced by \cite{stepp1986conceptual}, which is $\mathcal{NP}$-complete. Conceptual clustering distinguishes itself from traditional clustering methods through its emphasis on \textbf{symbolic, interpretable cluster descriptions}. In this framework, each data point is represented by Boolean-valued variables called \textbf{items}, and the dataset is partitioned into $\theta$ disjoint clusters that collectively cover all observations. Crucially, each cluster is not merely defined by implicit similarity measures or geometric centroids, but is explicitly \textbf{described}-or \textbf{covered}-by a symbolic representation consisting of a set of items, referred to as an \textbf{itemset} or \textbf{pattern}. This symbolic characterization provides transparent, human-understandable explanations of cluster membership, making conceptual clustering particularly relevant for \textbf{explainable AI} and \textbf{knowledge discovery} applications where interpretability is paramount. The data used in conceptual clustering are known as \textbf{transactional data} or structured data, where each transaction corresponds to a subset of the item universe.
Several  frameworks have been developed to solve the conceptual clustering problem, which can be classified into three main paradigms. The first category comprises \textbf{declarative methods} based on Boolean satisfiability (SAT) solving \cite{davidson2010sat,metivier2012constrained}, which encode the clustering task as a set of logical constraints. The second category consists of \textbf{constraint-based methods} that leverage constraint programming (CP) techniques \cite{6035705,laghzaoui2023constraint} to efficiently explore the solution space. The third category includes approaches based on \textbf{integer linear programming (ILP)} \cite{ouali2016efficiently,ouali2017integer,dao2018descriptive,DBLP:conf/ecai/HassineJKRG24}, which formulate conceptual clustering as a combinatorial optimization problem. ILP-based methods typically follow a two-phase pipeline: first, a set of candidate itemsets (patterns) is generated through pattern mining techniques; second, these patterns serve as input to an ILP formulation that selects an optimal subset to form the final cluster partition while satisfying coverage and disjointness constraints.
Recent advances in ILP-based conceptual clustering have introduced a generalized pattern model called \textbf{$k$-Relaxed Frequent Patterns ($k$-RFPs)} \cite{DBLP:conf/ecai/HassineJKRG24} aimed at improving clustering quality by relaxing strict itemset coverage requirements. Specifically, a new coverage concept called \textbf{$k$-cover} was defined, which allows transactions to be covered by patterns even when up to $k$ items from the pattern are absent in the transaction. This relaxation enables more flexible and robust cluster formation. The $k$-RFPs are first enumerated using a SAT solver, and subsequently fed into an ILP model that solves the clustering optimization problem. While this approach demonstrates improved clustering quality compared to classical closed patterns, we identified a significant computational bottleneck: \textbf{multiple distinct $k$-RFPs may induce identical $k$-covers}, leading to substantial redundancy in the candidate pattern space. This redundancy not only inflates the number of decision variables in the ILP formulation but also significantly increases solving time. 
To address this limitation, we propose an \textbf{Optimized Conceptual Clustering Method (OCCM)} that systematically eliminates redundant patterns prior to ILP-based cluster selection. Our key insight is that when multiple $k$-RFPs share the same $k$-cover, they represent equivalent symbolic descriptions from a clustering perspective, and thus only one representative pattern per distinct $k$-cover needs to be retained. 
Our contributions are threefold. \textbf{First}, we provide a formal theoretical analysis characterizing the conditions under which distinct $k$-RFPs induce identical $k$-covers, establishing a rigorous foundation for identifying redundancy. 
\textbf{Second}, we propose an efficient filtering algorithm that removes redundant patterns by retaining a single representative pattern for each distinct $k$-cover, thereby reducing the search space while preserving the semantic expressiveness of cluster descriptions. \textbf{Third}, we introduce theoretical measures to evaluate the interpretability and representativeness of selected patterns with respect to their induced clusters, providing a principled way to analyze pattern robustness and explanatory power. Finally, we conduct an extensive experimental evaluation on several real-world datasets to assess the effectiveness of the proposed optimization and to validate the relevance of the interpretability measures.

The remainder of this paper is organized as follows. Section \ref{sec2} introduces the formal notation and key definitions of conceptual clustering problem. Section \ref{sec3} presents our optimization framework, including a theoretical analysis of pattern redundancy, the proposed filtering algorithm, and the theoretical measures introduced to assess pattern interpretability and representativeness. Section \ref{sec4} presents an extensive experimental evaluation on real-world datasets, assessing computational performance, clustering quality, and the interpretability and representativeness of the selected patterns. Finally, Section \ref{sec5} concludes the paper and outlines future research directions.

\section{Formal notation}\label{sec2}

In this section we introduce some formal notation for the conceptual clustering problem. We first provide an overview of key concepts in pattern mining before introducing the principle of conceptual clustering.

Let $\mathcal{U}$ be a universe of symbols, also referred to as items, used to describe objects in the real world.  
For example, a flower may be characterized by Boolean attributes such as the presence of petals, fragrance, radial symmetry, etc. 
Individual elements of $\mathcal{U}$ are denoted by symbols such as $a, b, c$, etc.  
A \textbf{pattern} also called an itemset or classical pattern is defined as a non-empty subset of $\mathcal{U}$, denoted by $\mathcal{I} \subseteq \mathcal{U}$ with $\mathcal{I} \neq \emptyset$.  
The collection of all possible patterns is given by $2^{\mathcal{U}}$, and we use uppercase letters such as $\mathcal{I}, \mathcal{J}, \mathcal{Q}$ to denote the patterns.  

A \textbf{dataset} is a finite set of transactions, expressed as  
%
$\mathcal{D} = \{\tau_1, \tau_2, \dots, \tau_n\}$, where each $\tau_i$ $(i \in [1,\ldots,n])$ is a set of items called a \textbf{transaction}.

%
For a pattern $\mathcal{I}$ and a dataset $\mathcal{D}$, the \textbf{cover} of $\mathcal{I}$ consists of all transactions in $\mathcal{D}$ that contain every element of $\mathcal{I}$. Formally,  
$
\mathsf{Cov}(\mathcal{I}, \mathcal{D}) = \{ \tau_i \mid i \in [1\ldots n],  \tau_i \in \mathcal{D} \text{ and } \mathcal{I} \subseteq \tau_i \}.
$
%
The number of transactions in the cover is called the \textbf{support} of the pattern and is denoted by  
$
\mathsf{Sp}(\mathcal{I}, \mathcal{D}) = |\mathsf{Cov}(\mathcal{I}, \mathcal{D})|.
$  
A pattern $\mathcal{I}$ is considered  as \textbf{closed} if there is no other pattern $\mathcal{J}$ such that $\mathcal{J}$ properly contains $\mathcal{I}$ and has the same support, i.e.,  
$
\mathcal{I} \subset \mathcal{J} \implies \mathsf{Sp}(\mathcal{J}, \mathcal{D}) \neq \mathsf{Sp}(\mathcal{I}, \mathcal{D}).
$ 
\begin{definition}[Conceptual Clustering Problem]
    Let $\theta$ be a positive integer representing the number of clusters, and let $\mathcal{D}$ be a transactional dataset.  
    The goal is to find a partition of $\mathcal{D}$ into $\theta$ disjoint clusters covered by itemsets. Formally, determine a collection
    \[
        \mathcal{C} = \{\mathsf{Cov}(\mathcal{I}_{1}, \mathcal{D}), \mathsf{Cov}(\mathcal{I}_{2}, \mathcal{D}), \ldots, \mathsf{Cov}(\mathcal{I}_{\theta}, \mathcal{D})\},
    \]
    such that the following conditions hold:
    \begin{enumerate}
    \item Closure: All itemsets $\mathcal{I}_{1}, \dots, \mathcal{I}_{\theta}$ must be closed.

        \item Disjointness: $\forall i \neq j, \; \mathsf{Cov}(\mathcal{I}_{i}, \mathcal{D}) \cap \mathsf{Cov}(\mathcal{I}_{j}, \mathcal{D}) = \emptyset$,
        \item Completeness: $\bigcup_{i=1}^{\theta} \mathsf{Cov}(\mathcal{I}_{i}, \mathcal{D}) = \mathcal{D}$,
        \item Cardinality: $|\mathcal{C}| = \theta$.
    \end{enumerate}
\end{definition}
 In the following, we introduce the notions of  $k$-cover and $k$-support \cite{DBLP:conf/ecai/HassineJKRG24}.
\begin{definition}[ $k$-cover and $k$-support]
    Let $\mathcal{D} = \{\tau_1, \tau_2, \ldots, \tau_n\}$ be a dataset and let $k$ be a positive integer.  
    The \textbf{\emph{$k$-cover}} of a pattern $\mathcal{I}$ with respect to $\mathcal{D}$ is defined as:
    \[
        \mathsf{Cov}^{k}(\mathcal{I}, \mathcal{D}) = 
        \big\{ \tau_i \in \mathcal{D} \;\big|\; \tau_i \cap \mathcal{I} \neq \emptyset \;\wedge\; |\mathcal{I} \setminus \tau_i| \leq k \big\}.
    \]

    The \textbf{\emph{$k$-support}} of $\mathcal{I}$ in $\mathcal{D}$ is then given by the cardinality of its $k$-cover:
    \[
        \mathsf{Sp}^{k}(\mathcal{I}, \mathcal{D}) = \big| \mathsf{Cov}^{k}(\mathcal{I}, \mathcal{D}) \big|.
    \]
\end{definition}
We now formalize the notion of a $k$-Relaxed Frequent Pattern ($k$-RFP) \cite{DBLP:conf/ecai/HassineJKRG24}.
\begin{definition}[$k$-RFP]
Let $\mathcal{D}$ denote a dataset and let $\alpha > 0$ be a minimum support threshold.  
A pattern $\mathcal{I}$ is said to be a \textbf{\emph{$k$-Relaxed Frequent Pattern}} if and only if:
\[
    \mathsf{Sp}^{k}(\mathcal{I}, \mathcal{D}) \;\geq\; \alpha.
\]
\end{definition}

It is noteworthy that $k$-RFPs can be extracted using a satisfiability problem (SAT) approach.  
In this approach, a set of constraints is defined and encoded in conjunctive normal form (CNF), such that each satisfying assignment corresponds to a valid $k$-RFP.  
All solutions to the CNF are enumerated, with each solution representing a distinct $k$-RFP. 
\section{Optimized Conceptual Clustering Method (OCCM)} \label{sec3}

This section presents our optimized conceptual clustering method designed to reduce the number of itemsets  by discarding redundant patterns that yield identical covers.
\subsection{Pattern Filtering Strategy}

In the following, we provide a formal presentation of the issue introduced earlier. As a preliminary step, we recall an important property of classical closed itemsets concerning their cover.
\begin{proposition} \label{prop0}
Let ${\cal D}$ a transactional database and $\mathcal{I}$ and $\mathcal{J}$ two closed itemsets. Then, $\mathsf{Cov}(\mathcal{I}, \mathcal{D}) \neq \mathsf{Cov}(\mathcal{J}, \mathcal{D})$.  
\end{proposition}
\begin{proof}
Let $\mathcal{I}$ and $\mathcal{J}$ be two closed itemsets with $\mathcal{I} \neq \mathcal{J}$.  
Assume, for contradiction, that $\mathsf{Cov}(\mathcal{I}, \mathcal{D}) = \mathsf{Cov}(\mathcal{J}, \mathcal{D})$.  

By definition of closed itemsets, we have:
\[
\mathcal{I} = \bigcap_{\tau \in \mathsf{Cov}(\mathcal{I}, \mathcal{D})} \tau
\quad\text{and}\quad
\mathcal{J} = \bigcap_{\tau \in \mathsf{Cov}(\mathcal{J}, \mathcal{D})} \tau.
\]

Under the assumption that $\mathsf{Cov}(\mathcal{I}, \mathcal{D}) = \mathsf{Cov}(\mathcal{J}, \mathcal{D})$, it follows that:
\[
\mathcal{I} = \bigcap_{\tau \in \mathsf{Cov}(\mathcal{I}, \mathcal{D})} \tau \;
= \bigcap_{\tau \in \mathsf{Cov}(\mathcal{J}, \mathcal{D})} \tau
\;= \mathcal{J}.
\]

We have then : \[ \mathcal{I}= \mathcal{J}.\]

Which contradicts $\mathcal{I} \neq \mathcal{J}$.  
Therefore, two \textbf{distinct} closed itemsets must have \textbf{different} covers.
\end{proof}

Proposition \ref{prop0} states that, for a transactional database ${\cal D}$, no two distinct closed itemsets have the same cover.

It should be noted that, in the context of $k$-RFPs, this property \textbf{is not always true}. To illustrate this, consider the counterexample presented in Example \ref{examp1}.
\begin{example} \label{examp1}
    Let us consider the transaction database $\mathcal{D}$ of Table \ref{tab1} and suppose that $k=1$.
   Let consider two $1$-RFPs, $\mathcal{I}=\{a,b,c,h\}$ and $\mathcal{J}=\{a,b,c,d\}$.
We have $ \mathsf{Cov}^{1}(\mathcal{I}, \mathcal{D})  = \mathsf{Cov}^{1}(\mathcal{J}, \mathcal{D}) = \{ \tau_{1}, \tau_{3} \}.$
\begin{table}[h]
\centering
  \begin{tabular}{lllllllll}
     \hline
        Transactions & \multicolumn{8}{c}{Items} \\
        \hline
        $\tau_{1}$ & $a$ & $b$ & $c$ &  & $e$ &  &  &  \\
        $\tau_{2}$ &  &  &  &  & $e$ & $f$ &  &  \\
        $\tau_{3}$ & $a$ & $b$ &  & $d$ &  &  &  & $h$ \\
        $\tau_{4}$ &  &  &  &  &  &  & $g$ & $h$ \\
        \hline
    \end{tabular}
\caption{An illustrative dataset $\mathcal{D}$} \label{tab1}
\end{table}
\end{example}
As illustrated in Example \ref{examp1}, when $k > 0$, two distinct patterns may share the same cover, making one of them redundant.
If such cases occur frequently, a large number of unnecessary itemsets can be generated and subsequently processed by the ILP solver, leading to increased computational time.
To address this, we propose retaining only a single representative pattern per cover. This strategy reduces the number of generated $k$-RFPs and enhances the efficiency of the ILP solving process.
In the following, to demonstrate the existence of $k$-RFPs sharing the same $k$-cover, we first present  property for a specific condition of $k$-cover sets.
\begin{proposition} \label{prop1}
Let $\mathcal{H} \subseteq \mathcal{D}$ be a $k$-cover with $|\mathcal{H}| \geq 3$, let $\mathcal{U}_{\mathcal{H}} = \bigcup_{\tau \in \mathcal{H}} \tau$ denote its universe of items and let $m$ denote the number of distinct $k$-RFPs that share  $\mathcal{H}$.  
Suppose that  it always exists a maximal intersection:
\[
\mathcal{Q} = \bigcap_{\tau \in \mathcal{H}} \tau \neq \emptyset,
\]
and define $R = \mathcal{U}_{\mathcal{H}} \setminus \mathcal{Q}$, $ \quad \text{with} \quad \forall \; r \in R, \; r \in \tau \implies \forall \; \tau' \in \mathcal{H} \setminus \{\tau\} \; \text{we have} \;r \notin \tau'.$  
Then the number of distinct $k$-RFPs sharing $\mathcal{H}$ is
\[
m = \binom{|R|}{k} = \frac{|R|!}{k! \,(|R|-k)!},
\]
corresponding to patterns of the form $\mathcal{I} = \mathcal{Q} \cup E$, where $E \subseteq R$ and $|E| = k$.
\end{proposition}
%
%
\begin{proof}
By definition of a $k$-cover, any $k$-RFP denoted $\mathcal{I}$ with $\mathsf{Cov}^k(\mathcal{I},\mathcal{D}) = \mathcal{H}$ must contain all items of $\mathcal{Q}$, since $\mathcal{Q}$ appears in every transaction of $\mathcal{H}$.  
Because $\mathcal{I}$ may differ from each transaction in $\mathcal{H}$ by at most $k$ items, the only possible extensions of $\mathcal{Q}$ are obtained by adding exactly $k$ items from $R$. The condition that each item of $R$ belongs to exactly one transaction ensures that every such extension defines a distinct $k$-RFP.  
Thus, the number of distinct $k$-RFPs is exactly the number of ways to select $k$ items from $R$, i.e.,
\[
m = \binom{|R|}{k} = \frac{|R|!}{k!\,(|R|-k)!}.
\]
\end{proof}
%
%
%
%
%
%
%
%
%
%
%
%
%
%
Proposition \ref{prop1} states that given a $k$-cover set $\mathcal{H}$ with $|\mathcal{H}| \geq 3$, and assuming the existence of a maximal intersection $\mathcal{Q}$, i.e., the largest set of items common to all transactions in $\mathcal{H}$, we define $\mathcal{U}_{\mathcal{H}}$ as the set of all items appearing in $\mathcal{H}$ and $R = \mathcal{U}_{\mathcal{H}} \setminus \mathcal{Q}$ as the remaining items, with the condition that each item in $R$ appears in exactly one transaction. Any $k$-RFP covering $\mathcal{H}$ must include all items in $\mathcal{Q}$. To form a valid $k$-RFP, we select exactly $k$ items from $R$. Since items in $R$ occur in distinct transactions, each selection yields a unique $k$-RFP. Consequently, the total number of distinct $k$-RFPs sharing the cover $\mathcal{H}$ is exactly $\binom{|R|}{k}$.

To further clarify Proposition \ref{prop1}, we provide an illustrative example in Example \ref{exampp2}.
\begin{example}\label{exampp2}
Consider the transactional database $\mathcal{D}$ from Table \ref{tab11}. Let $k=1$, and select the 1-cover $\mathcal{H} = \{\tau_1, \tau_2, \tau_3\}$. We have $\mathcal{Q} = \{a,b,c\}$, 
$\mathcal{U}_{\mathcal{H}} = \{a,b,c,e,f,g\}$, and $R = \{e,f,g\}$.  
The number of 1-RFPs covering $\mathcal{H}$ is
\[
m = \binom{|R|}{1} = \frac{3!}{1! \,(3-1)!} = 3.
\] 
The corresponding 1-RFPs sharing $\mathcal{H}$ are $\{a,b,c,e\}$, $\{a,b,c,f\}$ and $\{a,b,c,g\}$.
\begin{table}[h]
\centering
  \begin{tabular}{lllllllllll}
        \hline
        Transactions & \multicolumn{10}{c}{Items} \\
        \hline
        $\tau_{1}$ & $a$ & $b$ & $c$ &  & $e$ &  &  & & & \\
        $\tau_{2}$ & $a$ & $b$ & $c$ &  &  & $f$ &  & & &   \\
        $\tau_{3}$ & $a$ & $b$ & $c$ &  &  &  & $g$ & & &  \\
        $\tau_{4}$ & $a$ &  &  &  &  &  &  & $h$ & $i$ &  \\
        $\tau_{5}$ & $a$ &  &  &  &  &  &  & $h$  &$i$ & $j$ \\
        \hline
    \end{tabular}
\caption{A representative dataset $\mathcal{D}$} \label{tab11}
\end{table}
\end{example}

Consequently, before introducing our method, we formally state the issue in Corollary \ref{cor1}. 
%
%
%
\begin{corollary} \label{cor1}
Let $k$ be a positive integer, and let $\mathcal{A} \subset 2^{\mathcal{U}}$ be a collection of $k$-RFPs.  
Then, it may exist $\mathcal{I}, \mathcal{J} \in \mathcal{A} \; $ s.t. $ \;$  
$\mathsf{Cov}^{k}(\mathcal{I}, \mathcal{D}) = \mathsf{Cov}^{k}(\mathcal{J}, \mathcal{D}). $
\end{corollary}
%
\begin{proof}
Examples \ref{examp1} and \ref{exampp2} serves as a concrete counterexamples, illustrating that distinct $k$-RFPs can share the same cover.
\end{proof}

%
Corollary \ref{cor1} states that, for a given set of relaxed frequent patterns, there may exist two patterns that share the same cover.

In what follows, we introduce our proposed approach. To provide context, we first summarize the main steps of the conceptual clustering process, as illustrated in Figure~\ref{fig1}.
\begin{figure}[h]
\begin{center}
\begin{tikzpicture}[node distance=2cm, >=stealth, thick]

\tikzstyle{process} = [rectangle, draw, rounded corners=2pt, 
                       minimum width=5cm, minimum height=1cm, align=center]

\node (step1) [process] {Generate $k$-RFPs  using SAT solver};
\node (step2) [process, below of=step1] {Extract clusters from 
$k$-RFPs via ILP};

\draw[->] (step1) -- (step2);

\end{tikzpicture}
\end{center}
\caption{Conceptual clustering approach: main steps}
\label{fig1}
\end{figure}
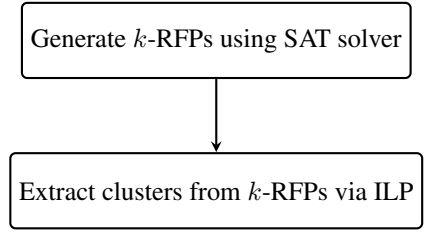

Our solution consists in eliminating redundant patterns, specifically those that share the same cover. This process effectively filters the pattern set. The overall optimized approach is illustrated in Figure~\ref{fig2}.
\begin{figure}[h]
\begin{center}
\begin{tikzpicture}[node distance=2cm, >=stealth, thick]

\tikzstyle{process} = [rectangle, draw, rounded corners=2pt, 
                       minimum width=5cm, minimum height=1cm, align=center]

\node (step1) [process] {Generate $k$-RFPs using SAT solver};
\node (step2) [process, below of=step1, fill=blue!20, text=black] {Filter redundant $k$-RFPs};
\node (step3) [process, below of=step2] {Extract clusters from $k$-RFPs via ILP};
%
%
\draw[->] (step1) -- (step2);
\draw[->] (step2) -- (step3);

\end{tikzpicture}
\end{center}
\caption{Optimized Conceptual clustering approach: main steps}
\label{fig2}
\end{figure}
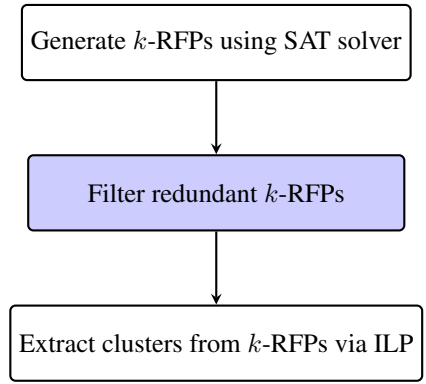

The proposed solution is formalized in the following definition.  
\begin{definition}[filtered patterns] \label{prop2}
Let $\mathcal{A} \subset 2^{\mathcal{U}}$ be a collection of $k$-RFPs.
The \textbf{filtered set} $\mathcal{A}_{f} \subset 2^{\mathcal{U}}$ is defined as : 
\[
\mathcal{A}_{f} =   \mathcal{A} \setminus \{\arg\min_{\mathcal{Q} \in \{\mathcal{I},\mathcal{J}\}} |\mathcal{Q}| \}\] 

where  

$ \forall \; \mathcal{I}, \mathcal{J} \in \mathcal{A}, \; \mathcal{I} \neq \mathcal{J}, \;  \mathsf{Cov}^{k}(\mathcal{I}, \mathcal{D}) = \mathsf{Cov}^{k}(\mathcal{J}, \mathcal{D}).$
 
%
%
\end{definition}

Definition \ref{prop2} establishes that when two $k$-RFPs share the same cover, only the maximal itemset is retained in the filtered pattern set $\mathcal{A}_{f}$. This choice is justified by interpretability, as larger itemsets provide more specific and descriptive characterizations of their covered transactions, and thus better represent the associated cluster.

In what follows, we formally present our filtration process in Algorithm~\ref{alg1}, which aims to eliminate redundant patterns by retaining a single representative pattern for each distinct cover. 
Given a set of candidate patterns $\mathcal{A}$ and a transactional dataset $\mathcal{D}$, the algorithm first sorts the patterns in increasing order of their size. It then iteratively computes, for each pattern $\mathcal{I} \in \mathcal{A}$, its associated $k$-cover $\mathsf{Cov}^{k}(\mathcal{I}, \mathcal{D})$. To ensure uniqueness of patterns with respect to their covers, the algorithm maintains a mapping, denoted by $\mathit{CoverMap}$, which associates each encountered $k$-cover with a single pattern. Formally, $\mathit{CoverMap}$ is modeled as a partial function $\mathit{CoverMap} : \mathcal{C} \rightharpoonup \mathcal{A}$, where $\mathcal{C}$ denotes the set of all possible $k$-covers and $\mathcal{A}$ is the set of candidate patterns. Whenever a newly processed pattern produces a $k$-cover that already appears in $\mathit{CoverMap}$, the corresponding entry is overwritten. Since patterns are processed in increasing order of size, this mechanism guarantees that only the largest pattern associated with each cover is retained. The filtered pattern set $\mathcal{A}_f$ is then defined as
\[
\mathcal{A}_f = \{ \mathit{CoverMap}(c) \mid c \in \mathrm{dom}(\mathit{CoverMap}) \} \subseteq \mathcal{A},
\]
where $\mathrm{dom}(\mathit{CoverMap})$ denotes the set of covers for which $\mathit{CoverMap}$ is defined. This procedure ensures that $\mathcal{A}_f$ contains exactly one representative pattern for each distinct $k$-cover while favoring patterns that provide more specific and descriptive characterizations of the covered transactions.

\begin{algorithm}[h]
\caption{Pattern Filtering Algorithm}
\label{alg1}
\begin{algorithmic}[1]
\REQUIRE $\mathcal{A}$, $\mathcal{D}$, $k$ 
\ENSURE $\mathcal{A}_f$ 

\STATE $\mathit{CoverMap} \gets \emptyset$ \COMMENT{Dictionary: $k$-cover $\to$ pattern}
\STATE $\mathcal{A}_{\text{sorted}} \gets \textsc{Sort}(\mathcal{A})$ \COMMENT{Smallest to largest by size}

\FOR{$\mathcal{I} \in \mathcal{A}_{\text{sorted}}$}
    \STATE $c \gets \mathsf{Cov}^{k}(\mathcal{I}, \mathcal{D})$ 
    \STATE $\mathit{CoverMap}[c] \gets \mathcal{I}$ \COMMENT{Overwrites if $c$ already exists}
\ENDFOR
\STATE $\mathcal{A}_f \gets \{ \mathit{CoverMap}(c) \mid c \in \mathrm{dom}(\mathit{CoverMap}) \}$

\RETURN $\mathcal{A}_f $
\end{algorithmic}
\end{algorithm}
\paragraph{Complexity Analysis :}
let $|\mathcal{A}|$ denote the number of $k$-RFPs, $|\mathcal{D}|$ the number of transactions, $|\mathcal{I}|_{\text{avg}}$ the average pattern length, and $|\tau|_{\text{avg}}$ the average transaction length. The algorithm first sorts $\mathcal{A}$ by pattern size in $O(|\mathcal{A}|\log|\mathcal{A}|)$ time. Computing the $k$-cover for each pattern $\mathcal{I} \in \mathcal{A}$ requires iterating over all transactions $\tau \in \mathcal{D}$ and checking whether $\mathcal{I}$ satisfies the $k$-cover conditions with respect to $\tau$, which incurs $O(|\mathcal{I}|_{\text{avg}} \cdot |\tau|_{\text{avg}})$ time per transaction. Therefore, the total time for computing $k$-covers for all patterns is 
$O(|\mathcal{A}| \cdot |\mathcal{D}| \cdot |\mathcal{I}|_{\text{avg}} \cdot |\tau|_{\text{avg}})$. Inserting each cover into the dictionary requires at most $O(|\mathcal{D}|)$ time per pattern, giving an overall cost of $O(|\mathcal{A}| \cdot |\mathcal{D}|)$. Combining these steps, the overall time complexity of the algorithm is then:
\[
O\Big(|\mathcal{A}|\log|\mathcal{A}| + |\mathcal{A}| \cdot |\mathcal{D}| \cdot |\mathcal{I}|_{\text{avg}} \cdot |\tau|_{\text{avg}}\Big).
\]

\subsection{Integer Linear Programming Formulation for Explainable Clustering}

We next introduce the integer linear programming (ILP) formulation underlying our approach. Before presenting our used model, we provide a brief overview of the general ILP framework. 

An ILP \cite{nemhauser1988integer} is an optimization problem in which a linear objective function is either maximized or minimized, subject to a set of linear constraints, with the additional requirement that all decision variables take integer values. ILPs have been widely applied in areas such as logistics, finance, manufacturing, and telecommunications. Solving an ILP consists of finding integer assignments to the decision variables that satisfy all constraints while optimizing the objective function. Modern solution methods, such as branch-and-bound, explore the solution space efficiently to determine optimal solutions \cite{nemhauser1988integer,ceria1998cutting}, and numerous specialized techniques have been proposed for various application domains \cite{papadomanolakis2007integer,cao2022trajectory,alhasnawi2025new,mohammed2023multi}.

Formally, let $\mathbf{x} \in \{0,1\}^m$ denote the vector of $m$ binary decision variables, and let $\mathbf{f} \in \mathbb{R}^m$ represent the vector of objective coefficients. Let $\mathbf{B} \in \mathbb{R}^{p \times m}$ be the matrix of constraint coefficients and $\mathbf{d} \in \mathbb{R}^{p}$ the vector of constraint bounds. Then, a general ILP can be written as

\begin{align*}
\text{maximize or minimize} \quad & \mathbf{f}^\top \mathbf{x}, \\
\text{subject to} \quad & \mathbf{B} \, \mathbf{x} \;\; (\leq, =, \text{ or } \geq) \;\; \mathbf{d}, \\
& x_i \in \{0,1\}, \quad i = 1, 2, \dots, m.
\end{align*}

In our work, we used the ILP model adopted by \cite{ouali2016efficiently,DBLP:conf/ecai/HassineJKRG24}, adapted to the formal notation introduced above. Formally, this model selects representative $k$-RFPs to cover the dataset while respecting the conceptual clustering constraints.

Let $\mathcal{A}_f$ denote the set of candidate closed $k$-RFPs extracted from $\mathcal{D}$, and let $p = |\mathcal{A}_f|$ and $m = |\mathcal{D}|$ is the number of transactions in the dataset.
We introduce a vector of binary decision variables
\[
\mathbf{x} = (x_\mathcal{I})_{\mathcal{I} \in \mathcal{A}_f}, \quad
x_\mathcal{I} =
\begin{cases}
1, & \text{if $\mathcal{I}$ is selected} \\
0, & \text{otherwise}
\end{cases}
\]

To encode the coverage of transactions by patterns, we define a binary matrix $B = (b_{\tau,\mathcal{I}}) \in \{0,1\}^{m \times p}$ such that
\[
b_{\tau,\mathcal{I}} =
\begin{cases}
1 & \text{if } \tau \in \mathsf{Cov}^k(\mathcal{I}, \mathcal{D}), \\
0 & \text{otherwise.}
\end{cases}
\]

Let $w_\mathcal{I}$ denote the weight of pattern $\mathcal{I}$, defined as the number of items in $\mathcal{I}$, i.e., $w_\mathcal{I} = |\mathcal{I}|$.
The ILP selects exactly $\theta$ representative $k$-RFPs while maximizing the sum of the sizes of the selected patterns. The model is formally presented as follows:

\begin{align*}
\text{maximize} \quad & \sum_{\mathcal{I} \in \mathcal{A}_f} w_\mathcal{I} \, x_\mathcal{I}, \\
\text{subject to : } \quad (1)  & \sum_{\mathcal{I} \in \mathcal{A}_f} b_{\tau,\mathcal{I}} \, x_\mathcal{I} = 1, \quad \forall \tau \in \mathcal{D}, \\ (2)
& \sum_{\mathcal{I} \in \mathcal{A}_f} x_\mathcal{I} = \theta, \\
& x_\mathcal{I} \in \{0,1\}, \quad \forall \;\mathcal{I} \in \mathcal{A}_f.
\end{align*}

Constraint (1) ensures that each transaction belongs to exactly one selected cluster, while constraint (2) enforces that exactly $\theta$ clusters are chosen. The objective maximizes the total size of the selected $k$-RFPs.

\subsection{Pattern-Based Explanations for Cluster Representation}
We recall that the ILP model employed in this work aims to determine an optimal partition into $\theta$ clusters by selecting one representative $k$-RFP for each cluster. 
However, an important \textbf{question} arises regarding the interpretability of the selected patterns: to what extent does a chosen pattern truly represent its induced cluster ? In particular, it is necessary to assess whether the selected pattern provides a reliable and meaningful description of it induced cluster.

To address this question, we evaluate the representativeness of a pattern from a cooperative contribution perspective using \textbf{Shapley values}. This approach allows us to quantify the contribution of each item within a $k$-RFP selected by the ILP model to the cluster it induces. In this framework, the items composing the pattern are treated as explanatory components, whose combined presence determines the extent to which the pattern accurately captures the structure of its associated cluster. 
By evaluating these contributions, we can assess the extent to which a pattern provides a faithful and interpretable representation of the cluster it covers.

First, to quantify the representativeness of a pattern with respect to its induced cluster, we define a measure called the \emph{importance of a pattern}, denoted $\mathrm{Imp}(\mathcal{I})$. For a selected $k$-RFP denoted $\mathcal{I}$, this measure computes, for each transaction in its $k$-cover set, the fraction of items of $\mathcal{I}$ that are present in that transaction, and then sums these fractions over all transactions in the cluster. Formally, the importance of $\mathcal{I}$ is defined as:
\[
\mathrm{Imp}(\mathcal{I}) = \sum_{\tau \in \mathsf{Cov}^k(\mathcal{I}, \mathcal{D})} \frac{|\mathcal{I} \cap \tau|}{|\mathcal{I}|}.
\]

This measure reflects how consistently the items of the pattern appear across the transactions of its induced cluster, with higher values indicating stronger representativeness. 

To evaluate the contribution of items within a selected pattern, we define a cooperative game in which each item of the pattern is considered as a player.

Let $\mathcal{I} \in \mathcal{A}_f$ be a $k$-RFP selected by the ILP model, and let $S \subseteq \mathcal{I}$ denote a non-empty subset of items forming a subpattern, i.e., $S \in 2^{\mathcal{I}} \setminus \{\emptyset\}$. The characteristic function $v(S)$ measures the representativeness of the subpattern $S$ with respect to the dataset using the importance measure introduced in this work. Formally, $v(S)$ is defined as follows:
\[
v(S)=\mathrm{Imp}(\mathcal{S}).
\]

Under this formulation, the Shapley value of an item $a \in \mathcal{I}$ is defined as:

\[
\phi(a) =
\sum_{S \subseteq \mathcal{I} \setminus \{a\}}
\frac{|S|!(|\mathcal{I}|-|S|-1)!}{|\mathcal{I}|!}
\left( v(S \cup \{a\}) - v(S) \right).
\]

The quantity $\phi(a)$ evaluates the average marginal contribution of item $a$ across all possible subsets of  the selected pattern $\mathcal{I}$. Intuitively, it quantifies how much the presence of $a$ improves the ability of the pattern to represent its associated cluster. A large Shapley value indicates that the item plays a central explanatory role, whereas a small value suggests that the item contributes little additional descriptive information.

To better understand the interpretability of the selected patterns, we propose two complementary measures capturing both the distribution of item contributions and the stability of the induced clusters. These measures rely on the Shapley values computed for items within a pattern and on the sensitivity of the cluster when items are removed.

\paragraph{Shapley Value Variance (SVV).}

While Shapley values quantify the individual contribution of each item to the representativeness of a pattern, it is also important to analyze how these contributions are distributed among the items composing the pattern. To this end, we define the \textit{Shapley value variance} (SVV), which measures the dispersion of item contributions inside a pattern.

Let $\mathcal{I}$ be a selected pattern, and let $\phi(a)$ denote the Shapley value associated with an item $a \in \mathcal{I}$. The SVV of $\mathcal{I}$ is defined as:

\[
\text{SVV}(\mathcal{I}) = \frac{1}{|\mathcal{I}|} \sum_{a \in \mathcal{I}} \left(\phi(a) - \overline{\phi}\right)^2
\]

where

\[
\overline{\phi} = \frac{1}{|\mathcal{I}|} \sum_{a \in \mathcal{I}} \phi(a),
\]
denotes the average Shapley value within the pattern.
This measure captures the heterogeneity of item contributions. A \textbf{low variance indicates that all items contribute similarly to the pattern representativeness}, suggesting a balanced and homogeneous structure. Conversely, a \textbf{high variance reveals that the pattern is dominated by a subset of highly influential items}, while other items contribute less to the cluster characterization.
\paragraph{Average Cluster Stability (ACS).}
To evaluate the robustness of a pattern with respect to the cluster it induces, we introduce the \textit{Average Cluster Stability} (ACS). This measure quantifies how much the cluster induced by a pattern changes when individual items are removed.

Let $C_{\mathcal{I}}$ denote the cluster induced by a pattern $\mathcal{I}$, i.e., $C_{\mathcal{I}} = \mathsf{Cov}^k(\mathcal{I}, \mathcal{D})$, and let $C_{\mathcal{I} \setminus \{a\}}$ denote the cluster induced after removing the item $a \in \mathcal{I}$, i.e., $C_{\mathcal{I} \setminus \{a\}} = \mathsf{Cov}^k(\mathcal{I} \setminus \{a\}, \mathcal{D})$.
For each item, we compute the \textbf{Jaccard Similarity} between the original cluster and the cluster obtained after removing the item. The ACS of pattern $\mathcal{I}$ is defined as:

\[
\text{ACS}(\mathcal{I}) = \frac{1}{|\mathcal{I}|} \sum_{a \in \mathcal{I}} 
J\left(C_{\mathcal{I}}, C_{\mathcal{I} \setminus \{a\}}\right)
\]
where $J\!\left(C_{\mathcal{I}}, C_{\mathcal{I} \setminus \{a\}}\right)$ denotes the Jaccard similarity between two clusters: the original cluster and the cluster induced after removing the item $a$ from the pattern $\mathcal{I}$. Formally,
\[
J\!\left(C_{\mathcal{I}}, C_{\mathcal{I} \setminus \{a\}}\right)
= \frac{C_{\mathcal{I}} \cap C_{\mathcal{I} \setminus \{a\}}}
{C_{\mathcal{I}} \cup C_{\mathcal{I} \setminus \{a\}}}.
\]

ACS quantifies the \textbf{average preservation of the cluster when items are removed from the pattern}. A \textbf{high ACS value} indicates that removing individual items has \textbf{limited impact on the cluster}, reflecting strong robustness and redundancy among pattern items. In contrast, a \textbf{low ACS value} suggests that items are individually critical for defining the cluster, meaning that the cluster structure is \textbf{highly sensitive} to item removal.

\paragraph{Complementarity of the Measures.}
The SVV and ACS provide complementary perspectives on pattern interpretability. The SVV evaluates how item contributions are distributed within the pattern, while ACS captures the overall stability of the cluster with respect to item removal. %

In this section, we provided a formal presentation of our approach. Its practical applicability is assessed through a series of experiments on real-world datasets, which will be described in detail in Section \ref{sec4}.
\section{Experiments}\label{sec4}

To evaluate the efficiency of our proposed approach, we performed an empirical study on several widely used real-world datasets, summarized in Table \ref{tab2}. We define $d$ as the density of a transactional dataset, representing the proportion of items present in each transaction relative to the total number of possible items. Higher density indicates that transactions contain many items, which increases the complexity of extracting frequent patterns.

\begin{table}[h]
\centering
    \begin{tabular}{llll}
    \hline
    Dataset& $|\mathcal{D}|$ & $|\mathcal{U}|$& $d$ (\%) \\
    \hline
    Lymph &148 & 68&40\\
    Mushroom &8124 & 119&18\\
    Primary-Tumor &336 & 31&48\\
    Soybean &630 & 50&32\\
    Tic-tac-toe &958 & 27&33\\
    Vote &435 & 48&33\\
   \hline
    \end{tabular}
  \caption{Real-world datasets }
\label{tab2}
\end{table}

Our experimental evaluation is conducted in three phases:

\begin{enumerate}
    \item \textbf{Pattern reduction.} In the first phase, we compare the number of $k$-RFPs before and after filtering redundant patterns that share the same $k$-cover. For each dataset, the minimum support threshold $\alpha$ is varied from $10\%$ to $40\%$. This allows us to assess whether redundant patterns persist under different support levels, noting that higher $\alpha$ values generally reduce the number of patterns due to increased selectivity.
    
    \item \textbf{ILP solving time and clustering quality.} In the second phase, we assess the impact of pattern filtering on ILP solving time by comparing our OCCM approach with the CCA-$k$-RFP-M1 method introduced by \cite{DBLP:conf/ecai/HassineJKRG24}. 
    To verify that clustering quality is preserved, we also evaluate the results using the F1-score with respect to ground-truth clusters. Unlike the ICS measure used by \cite{DBLP:conf/ecai/HassineJKRG24}, which considers only intra-cluster similarity, the F1-score provides a more comprehensive evaluation by directly comparing the obtained clusters with the true cluster assignments (ground-truth clusters).
    In this phase, for each dataset we selected the most suitable value of $\alpha$, aiming to maximize it in order to retain only the patterns necessary for extracting an optimal solution. The minimum value of $\alpha$ was fixed at $10\%$ across all datasets to avoid clustering timeouts.

    \item \textbf{Explainability Analysis of Pattern-Based Cluster Representations.}
In this final phase of the experiments, we evaluate how well each selected pattern represents the cluster it induces. Specifically, we analyze the stability and interpretability of patterns by examining item contributions and the sensitivity of the induced clusters to item removal. For each dataset, we compute the Shapley Value Variance (SVV) and the Average Cluster Stability (ACS) of the patterns discovered by the ILP model, and we assess the correlations between SVV and ACS, as well as between pattern size and ACS. This analysis provides insights into the representativeness of each pattern and the robustness of the corresponding clusters.

\end{enumerate}

It should be noted that, for generating the $k$-RFPs, we employed the same modified SAT solver introduced in \cite{DBLP:conf/ecai/HassineJKRG24}. Across all phases of our experiments, we adopted the same settings as in \cite{DBLP:conf/ecai/HassineJKRG24}. 
Specifically, the relaxation parameter was fixed to $k=1$, and the number of clusters $\theta$ was set to $2$. This choice is motivated by the ground-truth partitions of all considered datasets, which consist of exactly two clusters. Consequently, the clustering task can be interpreted as a binary classification problem. For this reason, we evaluate the clustering quality using the F1-score, a standard and widely adopted metric in binary classification settings.
For the experiments conducted in Phase II the maximum runtime was limited to one hour. 
A solution is considered not found if no optimal solution is obtained within the timeout, or if the solver continues searching past the time limit.

\subsection{Phase I: Comparison of Patterns Before and After Filtering}

In this phase, we compare the number of patterns obtained before and after the filtering step. The results of this comparison are summarized in Table \ref{tab3}.

The results reported in Table \ref{tab3} highlight the effectiveness of our filtering algorithm. Across all datasets and variations of the minimum support threshold $\alpha$, multiple patterns sharing identical covers are consistently detected. For each variation, we compute the percentage reduction (
$
\Delta = 100 \times \frac{|\mathcal{A}| - |\mathcal{A}_f|}{|\mathcal{A}|}
$
) observing reductions of up to $26.67\%$ in the \texttt{Tic-tac-toe} dataset. In all cases, we consistently obtain $|\mathcal{A}_f| < |\mathcal{A}|$, showing that the filtering step effectively removes redundancy while preserving cover diversity. These results empirically validate Proposition \ref{prop1} and Corollary \ref{cor1} from Section \ref{sec3}, confirming the existence of multiple patterns associated with the same $k$-cover.

\begin{table}[h]
%
%
%
   %
   \centering
\begin{tabular}{lllll}
\hline
    $\mathcal{D}$  &  $\alpha$  &  $|\mathcal{A}|$  &  $|\mathcal{A}_{f}|$ & $\Delta\;(\%)$
 \\
    \hline
    \multirow{4}{*}{Lymph}& $10\%$ &3605378 & \textbf{3349874} & 7.09 \\
     & $20 \%$& 759630& \textbf{719877} & 5.23 \\
     &$30\%$& 202602 & \textbf{190791}  & 5.83\\
     &$40 \% $&60470 & \textbf{55449} & 8.3\\
     \hline
     \multirow{4}{*}{Mushroom} & $10 \%$ & 128962& \textbf{118226}& 8.32 \\
     &$20\%$ & 19712 & \textbf{18176}& 7.79 \\
     &$30 \%$ & 4055 & \textbf{3726} & 8.11\\
     &$ 40 \%$ & 1135 & \textbf{1008} & 11.19 \\
     \hline
     \multirow{4}{*}{Primary-Tumor} & $10 \%$ & 256991& \textbf{243450} & 5.27 \\
     & $20\%$ & 76081 & \textbf{75220} & 1.13 \\
     & $30 \%$ & 30372 & \textbf{30313} & 0.19 \\
     & $40 \%$ & 14778 & \textbf{14764} & 0.09 \\
     \hline
     \multirow{4}{*}{Soybean} & $10 \%$ & 69191 & \textbf{68199} & 1.43  \\
     & $20\%$  & 11900 & \textbf{11664} & 1.98 \\
     & $30 \%$ & 3383 & \textbf{3257} & 3.72 \\
     & $40 \%$ & 1484 & \textbf{1400}  & 5.66 \\
     \hline
     \multirow{4}{*}{Tic-tac-toe} & $10\%$ & 4479 & \textbf{4453} & 0.58  \\
     & $20\%$  & 811 & \textbf{786} & 3.08  \\
& $30 \%$ & 171 & \textbf{154} & 9.94  \\
& $40 \%$ & 15 & \textbf{11} & 26.67 \\
\hline
\multirow{4}{*}{Vote}& $10\%$ & 280386 & \textbf{280179} & 0.07 \\
 & $20\%$  & 34098 & \textbf{34065} & 0.1 \\
& $30 \%$ & 6606 & \textbf{6576} & 0.45 \\
& $40 \%$ & 693 & \textbf{668} & 3.61 \\
     \hline
\end{tabular}
\caption{Comparison of the number of patterns before and after filtering.}

\label{tab3}
\end{table}

\subsection{Phase II: Impact of Filtering on ILP Solving Time and Clustering Quality}
In this phase, we evaluate the performance of our OCCM approach in comparison with the method proposed by \cite{DBLP:conf/ecai/HassineJKRG24} named CCA-$k$-RFP-M1. The results of this comparison are summarized in Table \ref{tab4}.

The results
reveal several important insights.
In terms of runtime performance, our optimized approach OCCM consistently achieves faster solving times across all datasets, confirming the effectiveness of eliminating redundant patterns prior to optimization. A single exception is observed for the \texttt{Tic-tac-toe} dataset, where the computation exceeded the time limit and no optimal solution was obtained. A similar behavior is observed for the CCA-$k$-RFP-M1 model, which is consistent with the high combinatorial complexity of this dataset and was therefore expected.
Regarding clustering quality, both approaches yield comparable results on three datasets (\texttt{Lymph}, \texttt{Soybean}, and \texttt{Vote}), demonstrating that redundancy removal does not degrade the quality of the resulting clusters. Interestingly, OCCM achieves higher clustering quality on the \texttt{Mushroom} and \texttt{Primary-Tumor} datasets. This improvement suggests that removing redundant patterns not only reduces computational complexity but can also enhance the selection of representative patterns, leading to more coherent cluster descriptions. This is due to the filtering strategy that prioritizes the largest itemset whenever multiple patterns share the same $k$-cover.
\begin{table*}[h]
\centering
\begin{tabular}{llllllll}
\hline
 \multirow{2}{*}{$\mathcal{D}$}  &  \multirow{2}{*}{$\alpha$}  &  \multicolumn{3}{c}{CCA-$k$-RFP-M1}  & \multicolumn{3}{c}{OCCM} \\
    \cmidrule(r){3-5}   \cmidrule(l){6-8}&
         & $|\mathcal{A}|$  & F1-score &  CPU time & $|\mathcal{A}_{f}|$ & F1-score & CPU time \\
 \hline
 Lymph  & $37 \%$ & 91888& \textbf{0.71} & 49.66  & \textbf{85470}& \textbf{0.71} & \textbf{29.74}\\
 Mushroom  & $20 \%$ & 19712& 0.34 & 3133.34  & \textbf{18176}& \textbf{0.73}& \textbf{1144.46} \\
 Primary-Tumor  & $25 \%$ & 45465& 0.25 & 55.88  & \textbf{45250}& \textbf{0.33}& \textbf{55.71} \\
 Soybean  & $20 \%$ & 11900& \textbf{0.29} & 18.29  & \textbf{11664} & \textbf{0.29} & \textbf{16.99} \\
 Tic-tac-toe  & $10 \%$ & 811& - & -  & \textbf{786}& -& - \\
 Vote  & $10 \%$ & 280386 & \textbf{0.51} & 1206.17  &  \textbf{280179}& \textbf{0.51} & \textbf{234.30} \\
 %
 \hline
\end{tabular}
\caption{ OCCM vs. CCA-$k$-RFP-M1. }
\label{tab4}
\end{table*}
\subsection{Phase III: Explainability Analysis of Pattern-Based Cluster Representations.}
In this phase, we evaluate the representativeness of patterns selected by the ILP model with respect to their induced clusters. Results are summarized in Table \ref{tabshapley} and graphically illustrated in Figures \ref{fig3} and \ref{fig4}. The \texttt{Tic-tac-toe} dataset was excluded, as no optimal solution was found.

In Figure \ref{fig3}, each dataset is represented by two patterns, illustrated as colored circles.
Each two circles of the same color correspond to the two representative patterns of a dataset. Pattern representativeness is assessed using Shapley Value Variance (SVV), which captures the dispersion of item contributions, and Average Cluster Stability (ACS), which measures the robustness of the induced cluster under item removal. Each circle represents a pattern positioned according to its corresponding pair of SVV and ACS values, i.e., $(SVV, ACS)$, while the circle size reflects the size of the pattern. 
For most datasets (\texttt{Lymph}, \texttt{Mushroom}, \texttt{Soybean}, \texttt{Vote}), we observe a negative correlation between SVV and ACS: patterns with more balanced item contributions (lower SVV) induce more stable clusters (higher ACS). 
The \texttt{Primary-Tumor} dataset appears to exhibit an opposite relationship between SVV and ACS. However, this behavior coincides with an increase in pattern size. Since pattern size consistently shows a strong positive correlation with ACS across all datasets, the observed increase in stability may be primarily explained by the larger pattern size rather than by the increase in SVV. This observation suggests that the stabilizing effect of pattern size can dominate the influence of contribution variability (SVV).

Figure \ref{fig4} shows a consistent positive correlation between pattern size and ACS across all datasets: larger patterns produce more stable clusters. This suggests that larger patterns provide a more complete description of cluster structure, making them less sensitive to item removal.
These findings validate our pattern selection strategy when multiple patterns share the same $k$-cover. Prioritizing larger patterns consistently increases the likelihood of selecting stable and representative patterns. Overall, a strong representative pattern is characterized by (i) preserving the induced cluster under item removal, and (ii) providing an expressive description of the cluster through a larger set of contributing items.

\begin{table}[h]
\centering
\begin{tabular}{llll}
\hline
$\mathcal{D}$&  Pattern size & SVV & ACS \\
\hline
\multirow{2}{*}{Lymph} & $|\mathcal{I}_1| = 4 $ & 48.89 & 0.84 \\
& $|\mathcal{I}_2| =  13  $ & 2.22 & 0.95 \\
\hline
\multirow{2}{*}{Mushroom} & $|\mathcal{I}_1| = 5 $ & 146486.32 & 0.84 \\
& $|\mathcal{I}_2| = 8 $ & 27367.52 & 0.86 \\
\hline
\multirow{2}{*}{Primary-Tumor} & $|\mathcal{I}_1| = 2 $ & 12.5 & 0.53 \\
& $|\mathcal{I}_2| = 4 $ & 319.02 & 0.75 \\
\hline
\multirow{2}{*}{Soybean}  & $|\mathcal{I}_1| = 3 $ & 4742.43 & 0.70 \\
& $|\mathcal{I}_2| = 9 $ & 6.26 & 0.97 \\
\hline
\multirow{2}{*}{Vote} & $|\mathcal{I}_1| = 2 $ & 5724.5 & 0.53 \\
& $|\mathcal{I}_2| = 3 $ & 1196.43 & 0.58 \\
\hline
\end{tabular}
\caption{Pattern size, SVV, and ACS values for each dataset.}
\label{tabshapley}
\end{table}

\begin{figure}[h!]
    \begin{center}
    \includegraphics[width=0.88\linewidth]{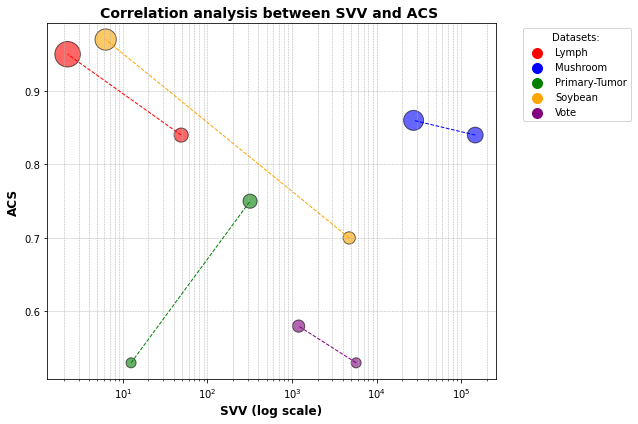}
    \end{center}
    \caption{Correlation between Shapley Value Variance (SVV) and Average Cluster Stability (ACS).}
    \label{fig3}
\end{figure}
\begin{figure}[h!]
    \begin{center}
    \includegraphics[width=0.88\linewidth]{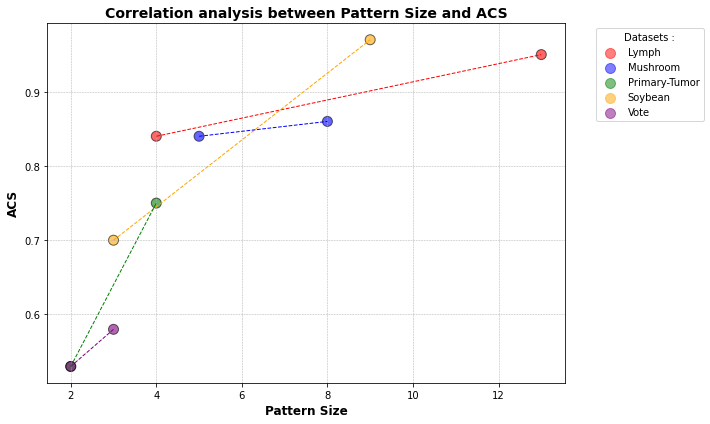}
    \end{center}
    \caption{Influence of Pattern Size on Cluster Stability.}
    \label{fig4}
\end{figure}
\section{Conclusion and future work}\label{sec5}
In this paper, we introduced OCCM, an optimized conceptual clustering method designed to eliminate redundancy in the generation of $k$-RFPs. Our work is motivated by the theoretical observation that distinct $k$-RFPs may share the same $k$-cover, which leads to redundant symbolic descriptions and increases the computational complexity of the ILP solver during the clustering phase. To address this limitation, we proposed a filtering strategy that retains a single representative pattern for each cover while favoring the largest itemset. This choice is motivated by promoting richer and more informative cluster descriptions, which is essential in explainable and knowledge-driven clustering settings.
Beyond optimization, we conducted a detailed analysis of the ILP output to evaluate the representativeness and stability of the selected patterns with respect to their induced clusters. To this end, we quantified the contribution of individual items using Shapley values and extended this analysis to itemsets through the SVV and ACS measures. These measures provide complementary insights into the distribution of item contributions and the robustness of clusters to item removal, allowing us to better understand the explanatory power of selected patterns.
Extensive experiments on several real-world datasets validate the effectiveness of the proposed approach. The results confirm the frequent presence of redundant patterns and demonstrate that their elimination substantially reduces the number of candidate patterns, leading to improved ILP solving times without compromising clustering quality. Furthermore, the experimental interpretability analysis reveals a consistent relationship between pattern size and cluster stability, showing that larger patterns tend to provide more representative and robust cluster descriptions. This empirical observation supports our design choice of prioritizing larger itemsets when selecting representative patterns.
This work opens several promising research directions. First, we aim to incorporate redundancy-awareness directly into the SAT-based generation process of $k$-RFPs, rather than relying solely on post-processing. Although we previously investigated constraints addressing this issue, their encoding complexity significantly limited solver scalability. Future work will therefore focus on designing more compact and efficient formulations. Second, we plan to extend the ILP objective function by integrating interpretability-oriented criteria, since maximizing pattern size alone does not fully capture explanatory quality. %

\bibliographystyle{kr}


\end{document}